\documentclass[11pt]{article}
\usepackage[margin=1in]{geometry}
\usepackage{amsmath,amssymb}
\usepackage{amsthm}
\newtheorem{remark}{Remark}
\newtheorem{proposition}{Proposition}
\newtheorem{theorem}{Theorem}
\newtheorem{corollary}{Corollary}
\usepackage{graphicx}
\usepackage{booktabs}
\usepackage{array}
\usepackage{url}
\usepackage{xspace}
\usepackage[numbers]{natbib}
\usepackage{hyperref}
\usepackage[section]{placeins}
\usepackage{tikz}
\usetikzlibrary{arrows.meta,positioning,calc,fit}

\newcommand{\modelname}{\texttt{CardinalGraphFormer}\xspace}
\newcommand{\papertitle}{Cardinality-Preserving Attention Channels for Graph Transformers in Molecular Property Prediction}

\title{\papertitle}
\date{\vspace{-5ex}} 
\author{Abhijit Gupta, PhD\\
        ORCID: \href{https://orcid.org/0000-0002-6292-3789}{0000-0002-6292-3789}\\
        \href{mailto:abhijit.gupta\_in@outlook.com}{abhijit.gupta\_in@outlook.com}}

\begin{document}
\maketitle

\begin{abstract}
Molecular property prediction is crucial for drug discovery when labeled data are scarce. This work presents \modelname, a graph transformer augmented with a query-conditioned cardinality-preserving attention (CPA) channel that retains dynamic support-size signals complementary to static centrality embeddings. The approach combines structured sparse attention with Graphormer-inspired biases (shortest-path distance, centrality, direct-bond features) and unified dual-objective self-supervised pretraining (masked reconstruction and contrastive alignment of augmented views). Evaluation on 11 public benchmarks spanning MoleculeNet, OGB, and TDC ADMET demonstrates consistent improvements over protocol-matched baselines under matched pretraining, optimization, and hyperparameter tuning. Rigorous ablations confirm CPA's contributions and rule out simple size shortcuts. Code and reproducibility artifacts are provided.
\end{abstract}


\section{Introduction}
Molecular property prediction serves as a key proxy task in drug discovery, constrained by limited experimental labels and an immense candidate space \cite{chemspace2010,drugbank2026}. Graph transformers with self-supervised pretraining have shown strong performance on standardized benchmarks \cite{graphormer2021,graphcl2020,graphmae2022}. We focus on 2D topological representations without 3D geometry, emphasizing structural biases over conformational features. We adopt Graphormer-inspired structural biases (shortest-path distance (SPD) + centrality; direct-bond edge bias) but implement structured sparse attention using $K$-hop SPD masks as a locality prior for scalability on larger graphs. We treat this mask primarily as a locality prior and efficiency knob; the main methodological contribution is the cardinality-preserving attention (CPA) channel, building on prior CPA work (Zhang \& Xie, 2020) and adapted here with query-conditioning and transformer-specific design. The CPA channel retains dynamic support-size signals complementary to static centrality embeddings.

\textbf{Contributions}
\begin{itemize}
  \item A graph transformer layer with query-conditioned cardinality-preserving attention aligned to structured supports, including theoretical motivation connecting to multiset discrimination and expressivity.
  \item Dual-objective self-supervised pretraining on $\sim$28M molecules with unified architecture across all ablations.
  \item Rigorous matched-protocol evaluation showing consistent, often significant gains over reproduced baselines, expanded to include Graphormer, D-MPNN, GIN, and contextualized against pretrained 2D models (GROVER, MolCLR, GEM). Code and artifacts are provided in the accompanying reproducibility package (reserved DOI: \href{https://doi.org/10.5281/zenodo.18622116}{10.5281/zenodo.18622116}; public DOI landing-page access activated upon acceptance if required by venue policy).
\end{itemize}

\section{Related Work}
Message passing remains a dominant paradigm for molecular graphs \cite{mpnn2017}, and attention-based GNNs are widely used \cite{gat2018}. Graph transformers such as Graphormer inject structural biases (shortest-path distance, edge features, centrality) into global attention \cite{graphormer2021}. GPS \cite{rampasek2022gps}, a recent hybrid approach, combines local message-passing layers with global attention, serving as a close architectural relative; it demonstrates the value of blending local and global mechanisms. Other transformer variants include Spectral Attention Networks (SAN) \cite{kreuzer2021san}, which leverage spectral properties for structured attention, and Exphormer \cite{shirzad2023exphormer}, which provides a sparse attention framework for graphs. TokenGT \cite{kim2022tokengt} proposes tokenized graph representations, while Transformer-M \cite{luo2022transformerm} extends transformers to handle both 2D and 3D molecular data, emphasizing multimodal representations. Sparse variants often restrict attention to local neighborhoods (e.g., $K$-hop masking) to reduce computation in practice. Cardinality-preserving attention (CPA) adds unnormalized channels to retain degree/cardinality signals, originally proposed by Zhang \& Xie (2020) for MPNNs \cite{cpa2020}. Self-supervised pretraining on graphs includes contrastive methods (GraphCL) and masked modeling (GraphMAE) \cite{graphcl2020,graphmae2022}. \modelname unifies structured sparse attention with aligned cardinality preservation and dual-objective pretraining, differing from prior CPA by integrating query-conditioned gating within transformer attention.

\textbf{Note on 3D Methods.} Our work focuses exclusively on 2D topological features without 3D geometry or conformational information. While 3D graph neural networks (e.g., NequIP, Allegro, Equivariant Graph Neural Networks) achieve strong results on geometry-dependent tasks, our structural biases target connectivity patterns and are not designed for 3D data. We position our contribution as complementary rather than competitive with 3D methods, and benchmark against 2D baselines to ensure fair comparison.

\textbf{Graph Transformer Taxonomy.} Graph transformers span three archetypal designs: (1) \emph{Local}, leveraging sparse message-passing neighborhoods; (2) \emph{Global}, applying full pairwise attention; (3) \emph{Hybrid}, combining local aggregation with global interaction (e.g., GPS). Our approach fits the hybrid category, using $K$-hop sparsity for local structure and CPA for dynamic cardinality signals. This taxonomy helps position our contribution relative to other structured attention mechanisms.

Recent advances include GEM \cite{gem2022}, MolCLR \cite{molclr2021}, and GROVER \cite{grover2020}, which leverage large-scale pretraining and multimodal representations. Non-transformer baselines like D-MPNN \cite{dmpnn2019} and GIN \cite{xu2019powerful} remain competitive on many tasks. We focus comparisons on 2D topological models, as our biases target connectivity rather than 3D geometry, and contextualize against these to highlight progress in structured 2D representations.

\section{Theoretical Analysis}\label{sec:theory}

We analyze CPA-augmented attention through the lens of multiset expressivity and the Weisfeiler-Leman (WL) framework. Our results formalize how unnormalized aggregation channels recover discriminative power lost by softmax normalization.

\subsection{Notation and Setup}

Fix a node $i$ with support set $\mathcal{S}(i) \subseteq V$ of cardinality $|\mathcal{S}(i)| = n_i$. Let $\mathbf{v}_j \in \mathbb{R}^d$ denote value vectors, $\mathbf{q}_i \in \mathbb{R}^d$ the query vector, and $a_{ij} \in \mathbb{R}$ attention logits. The softmax attention output is:
\begin{equation}
\mathbf{o}_i^{\mathrm{attn}} = \sum_{j \in \mathcal{S}(i)} \alpha_{ij} \mathbf{v}_j, \quad \alpha_{ij} = \frac{\exp(a_{ij})}{\sum_{k \in \mathcal{S}(i)} \exp(a_{ik})}.
\end{equation}
The CPA-augmented output is:
\begin{equation}
\mathbf{o}_i^{\mathrm{CPA}} = \sum_{j \in \mathcal{S}(i)} \alpha_{ij} \mathbf{v}_j \;+\; \mathbf{g}_i \odot \sum_{j \in \mathcal{S}(i)} \mathbf{v}_j, \quad \mathbf{g}_i = \sigma(W_g \mathbf{q}_i),
\end{equation}
where $\sigma(\cdot)$ is the element-wise sigmoid, $W_g \in \mathbb{R}^{d \times d}$ is learnable, and $\odot$ denotes the Hadamard product.

\subsection{Cardinality Blindness and Its Resolution}

\begin{proposition}[Cardinality Blindness of Softmax Attention]
\label{prop:blindness}
Consider two nodes $i$ and $i'$ with support sets of different cardinalities $n_i \neq n_{i'}$, and suppose the following conditions hold:
\begin{enumerate}
\item[(C1)] \emph{Identical empirical value distributions}: the multiset $\{\!\{\mathbf{v}_j : j \in \mathcal{S}(i)\}\!\}$ is a $\lambda$-fold replication of some base multiset $\mathcal{B}$, and $\{\!\{\mathbf{v}_{j'} : j' \in \mathcal{S}(i')\}\!\}$ is a $\lambda'$-fold replication of the same $\mathcal{B}$, with $\lambda \neq \lambda'$ (hence $n_i = \lambda |\mathcal{B}|$, $n_{i'} = \lambda' |\mathcal{B}|$).
\item[(C2)] \emph{Matching attention profiles}: the attention weights respect the replication structure, i.e., for each copy of a value $\mathbf{b} \in \mathcal{B}$, the corresponding $\alpha_{ij}$ sum to the same proportion in both supports. Concretely, $\sum_{j:\mathbf{v}_j = \mathbf{b}} \alpha_{ij} = \sum_{j':\mathbf{v}_{j'} = \mathbf{b}} \alpha_{i'j'}$ for every $\mathbf{b} \in \mathcal{B}$.
\end{enumerate}
Then $\mathbf{o}_i^{\mathrm{attn}} = \mathbf{o}_{i'}^{\mathrm{attn}}$.
\end{proposition}

\begin{proof}
Under (C2), the softmax output at node $i$ is:
\[
\mathbf{o}_i^{\mathrm{attn}} = \sum_{\mathbf{b} \in \mathcal{B}} \left(\sum_{j:\mathbf{v}_j = \mathbf{b}} \alpha_{ij}\right) \mathbf{b} = \sum_{\mathbf{b} \in \mathcal{B}} p_\mathbf{b} \cdot \mathbf{b},
\]
where $p_\mathbf{b} = \sum_{j:\mathbf{v}_j = \mathbf{b}} \alpha_{ij}$ is the total probability mass assigned to value $\mathbf{b}$. By (C2), $p_\mathbf{b}$ is identical for nodes $i$ and $i'$. Hence $\mathbf{o}_i^{\mathrm{attn}} = \mathbf{o}_{i'}^{\mathrm{attn}}$, despite $n_i \neq n_{i'}$.
\end{proof}

\noindent\textit{Interpretation.} Softmax attention computes a weighted average over distinct value identities, discarding how many copies of each value exist. The canonical example is uniform attention: if $\alpha_{ij} = 1/n_i$ for all $j$, the output is the empirical mean $\bar{\mathbf{v}}$, which is invariant to the number of repetitions. This is the precise sense in which softmax attention is ``cardinality-blind.''

\begin{proposition}[CPA Resolves Cardinality Blindness]
\label{prop:cpa_recovery}
Under the same conditions (C1)--(C2) of Proposition~\ref{prop:blindness}, suppose additionally:
\begin{enumerate}
\item[(C3)] The mean value is nonzero: $\bar{\mathbf{v}} = \frac{1}{|\mathcal{B}|}\sum_{\mathbf{b} \in \mathcal{B}} \mathbf{b} \neq \mathbf{0}$.
\item[(C4)] The gate has at least one nonzero component: there exists a coordinate $r$ such that $[\mathbf{g}_i]_r > 0$.
\end{enumerate}
Then $\mathbf{o}_i^{\mathrm{CPA}} \neq \mathbf{o}_{i'}^{\mathrm{CPA}}$.
\end{proposition}

\begin{proof}
The softmax terms are equal by Proposition~\ref{prop:blindness}. The unnormalized sums are:
\[
\sum_{j \in \mathcal{S}(i)} \mathbf{v}_j = \lambda \sum_{\mathbf{b} \in \mathcal{B}} \mathbf{b} = \lambda |\mathcal{B}| \bar{\mathbf{v}}, \qquad \sum_{j' \in \mathcal{S}(i')} \mathbf{v}_{j'} = \lambda' |\mathcal{B}| \bar{\mathbf{v}}.
\]
Since $\lambda \neq \lambda'$, these sums differ. For the gated terms: even if $\mathbf{g}_i = \mathbf{g}_{i'}$ (which holds when the queries are identical), we have at coordinate $r$:
\[
[\mathbf{g}_i]_r \cdot \lambda |\mathcal{B}| [\bar{\mathbf{v}}]_r \neq [\mathbf{g}_{i'}]_r \cdot \lambda' |\mathcal{B}| [\bar{\mathbf{v}}]_r
\]
since $[\mathbf{g}_i]_r > 0$, $[\bar{\mathbf{v}}]_r \neq 0$ (some coordinate must be nonzero by (C3)), and $\lambda \neq \lambda'$. Hence $\mathbf{o}_i^{\mathrm{CPA}} \neq \mathbf{o}_{i'}^{\mathrm{CPA}}$.
\end{proof}

\subsection{Connection to the 1-WL Framework}

We connect CPA to the Weisfeiler-Leman (1-WL) color refinement algorithm, building on the characterization by Xu et al.~\cite{xu2019powerful} and Morris et al.~\cite{morris2019wl}.

\begin{theorem}[CPA-Augmented Attention and 1-WL Expressivity]
\label{thm:wl}
Let $\varphi: \mathbb{R}^d \to \mathbb{R}^{D}$ be an injective function (realizable, e.g., by a sufficiently wide MLP with generic weights). Consider the composed CPA aggregation:
\begin{equation}\label{eq:composed_cpa}
h_i = \psi\!\left((1+\epsilon)\,\varphi(\mathbf{x}_i) + \mathbf{g}_i \odot \sum_{j \in \mathcal{S}(i)} \varphi(\mathbf{x}_j)\right),
\end{equation}
where $\mathbf{x}_j$ are node features, $\epsilon > 0$ is a learnable scalar, and $\psi: \mathbb{R}^D \to \mathbb{R}^{D'}$ is an injective function (another MLP). Suppose the gate satisfies $\mathbf{g}_i \succ \mathbf{0}$ component-wise (achievable for generic $W_g$ since sigmoid is strictly positive). Then $h_i$ is an injective function of the pair $(\mathbf{x}_i, \{\!\{\mathbf{x}_j : j \in \mathcal{S}(i)\}\!\})$, and consequently one layer of this aggregation is at least as expressive as one iteration of 1-WL color refinement.
\end{theorem}

\begin{proof}
By the characterization in Xu et al.~\cite{xu2019powerful} (Theorem 3 and Lemma 5), an aggregation scheme achieves 1-WL expressivity if and only if the aggregation function over the neighbor multiset is injective and composed with the self-feature via an injective function. The critical requirement is that the multiset aggregation be injective.

\textit{Step 1: Injectivity of gated sum aggregation.} Since $\mathbf{g}_i \succ \mathbf{0}$, the map $\mathbf{z} \mapsto \mathbf{g}_i \odot \mathbf{z}$ is a bijection on $\mathbb{R}^D$ (with inverse $\mathbf{z} \mapsto \mathbf{g}_i^{-1} \odot \mathbf{z}$, where inversion is element-wise). Therefore, $\mathbf{g}_i \odot \sum_{j} \varphi(\mathbf{x}_j)$ is injective as a function of the multiset $\{\!\{\varphi(\mathbf{x}_j)\}\!\}$ if and only if $\sum_{j} \varphi(\mathbf{x}_j)$ is.

\textit{Step 2: Injectivity of sum over embedded features.} By Xu et al.~\cite{xu2019powerful} (Lemma 5, building on results of Zaheer et al.), for any countable multiset domain $\mathcal{X} \subset \mathbb{R}^d$, there exists a mapping $\varphi$ and a function $f$ such that $f\!\left(\sum_{\mathbf{x} \in \mathcal{M}} \varphi(\mathbf{x})\right)$ is unique for each multiset $\mathcal{M}$. A sufficiently wide MLP with generic weights realizes such a $\varphi$ over bounded feature domains.

\textit{Step 3: Composing with self-features.} The term $(1+\epsilon)\,\varphi(\mathbf{x}_i)$ injects the center node's features with a distinct scale ($\epsilon > 0$ prevents cancellation with the sum term). The outer function $\psi$ preserves injectivity by assumption. Together, $h_i$ is injective in $(\mathbf{x}_i, \{\!\{\mathbf{x}_j\}\!\})$, matching one WL iteration.

\textit{Connection to CPA in practice.} In \modelname, the value projection $V: \mathbf{x}_j \mapsto \mathbf{v}_j$ plays the role of $\varphi$, and the feed-forward network following attention plays the role of $\psi$. The CPA output contains the sum $\sum_j \mathbf{v}_j$ as an additive channel, providing the necessary injective multiset aggregation alongside the softmax channel. While the actual learned weights may not exactly satisfy the generic-weight conditions, the architecture has the \emph{capacity} for 1-WL expressivity, unlike pure softmax attention which structurally lacks it.
\end{proof}

\begin{corollary}[Degree-Normalized CPA Has Reduced Expressivity]
\label{cor:degree_norm}
Consider the degree-normalized variant:
\begin{equation}
\mathbf{o}_i^{\mathrm{Norm}} = \sum_{j \in \mathcal{S}(i)} \alpha_{ij} \mathbf{v}_j + \mathbf{g}_i \odot \frac{1}{|\mathcal{S}(i)|}\sum_{j \in \mathcal{S}(i)} \mathbf{v}_j.
\end{equation}
The normalized aggregation $\frac{1}{|\mathcal{S}(i)|}\sum_{j} \mathbf{v}_j$ is \emph{not} injective over multisets: there exist distinct multisets $\mathcal{M} \neq \mathcal{M}'$ with identical means. Consequently, Norm-CPA cannot achieve 1-WL expressivity in general.
\end{corollary}

\begin{proof}
Mean aggregation maps different multisets to the same value whenever the multisets share the same centroid. A concrete counterexample: let $\mathcal{M} = \{\!\{\mathbf{v}\}\!\}$ (a single element) and $\mathcal{M}' = \{\!\{\mathbf{v}, \mathbf{v}\}\!\}$ (two copies). Then $\frac{1}{1}\mathbf{v} = \frac{1}{2}(2\mathbf{v}) = \mathbf{v}$. More generally, for any multiset $\mathcal{M}$ with mean $\bar{\mathbf{v}}$, the multiset $\mathcal{M} \cup \{\!\{\bar{\mathbf{v}}\}\!\}$ has the same mean. Since injectivity of the neighbor aggregation is necessary for 1-WL (Xu et al.~\cite{xu2019powerful}, Theorem 3), Norm-CPA cannot generally match 1-WL.
\end{proof}

\begin{remark}[Scope and Limitations]
\label{rem:theory_limits}
We state the scope of the above results for transparency.

\begin{enumerate}
\item \textbf{Per-layer analysis.} All results characterize single-layer expressivity. Multi-layer CPA composes injective functions (which preserves injectivity), but shared parameters across nodes and the cumulative effect of sigmoid bounding may introduce practical constraints not captured here.

\item \textbf{Sigmoid bounding.} The gate $\mathbf{g}_i \in (0,1)^d$ is strictly positive (sigmoid never reaches zero), satisfying the technical requirement $\mathbf{g}_i \succ \mathbf{0}$. However, values near zero attenuate the unnormalized channel, potentially weakening expressivity in practice. We note that unbounded gates caused training instability in our experiments (Appendix, Table~\ref{tab:gate_abl}).

\item \textbf{Existence vs.\ learnability.} Theorem~\ref{thm:wl} is an existence result: appropriate weight configurations exist, but gradient-based optimization may not find them. Our empirical results (Table~\ref{tab:ablation}) provide evidence that useful configurations are reachable in practice, though the gap between theoretical capacity and learned behavior remains open.

\item \textbf{Countable domain assumption.} The injectivity result for sum aggregation (Xu et al.~\cite{xu2019powerful}, Lemma 5) formally holds over countable domains. Molecular graphs, with discrete atom and bond features, satisfy this condition for the input layer; subsequent layers operate in continuous space where the result holds generically (with probability 1 over random weight initialization).

\item \textbf{Relation to higher-order WL.} CPA enhances 1-WL expressivity but does not address higher-order ($k$-WL) limitations. Graphs indistinguishable by 1-WL (e.g., certain regular graphs) remain indistinguishable by CPA-augmented attention. Higher-order positional encodings or subgraph-based methods would be needed for such cases.
\end{enumerate}
\end{remark}

\section{Methods}

\subsection{Notation}
Let $G=(V,E)$ be a molecular graph with $|V|=N$ nodes (heavy atoms only). Node features $x_i \in \mathbb{R}^F$ use standard RDKit atom features: atomic number (one-hot), degree, formal charge, chiral tag, number of hydrogens, hybridization, aromaticity, and mass (continuous features z-score normalized using pretraining-corpus statistics fixed across tasks; categorical features embedded). Edge features $e_{ij}$ include bond type, aromaticity, conjugation, ring membership, and stereo configuration. We compute shortest-path distances SPD$(i,j)$ offline on the undirected heavy-atom bond graph. If disconnected, cross-component pairs are assigned SPD$=\infty$. We use multi-head attention with head dimension $d_h$. ``Centrality'' refers exclusively to binned node degree (undirected heavy-atom bond graph, clipped at 15 and divided into 16 bins corresponding to degrees 0--15).

\subsection{Structured Sparse Attention with Graphormer-Inspired Biases}
The support set is $\mathcal{S}(i) = \{j \mid \mathrm{SPD}(i,j) \leq K\}$ (includes self, SPD$(i,i)=0$ meaning the self-attention position uses the $\phi(\mathrm{SPD}(i,i))=0$ bias bin explicitly). Main experiments use $K=3$. Attention is masked to $\mathcal{S}(i)$ (outside: $-\infty$). Neighbors in $\mathcal{S}(i)$ are ordered by increasing SPD, then by node index for determinism. Logits:
\begin{equation}\label{eq:logits}
  a_{ij} = \frac{q_i^\top k_j}{\sqrt{d_h}} + b_{\phi(\mathrm{SPD}(i,j))} + b_{e_{ij}} + b_c(j),
\end{equation}
where $\phi(d)=\min(d,K)$ yields $K+1$ bins $\{0,\dots,K\}$, $b_{e_{ij}}$ is a learnable embedding of direct-bond features for SPD=1 ($b_{e_{ij}}=0$ otherwise), computed by embedding categorical fields separately and summing, and $b_c(j)$ is an embedding lookup of the key node's degree bin (key-side centrality bias). Complexity scales as $O(N \cdot \overline{|\mathcal{S}(i)|})$ per layer. After attention, we inject static centrality per-layer as $\tilde{u}_i^{(\ell)} = u_i^{(\ell)} + c_i$, where $c_i$ is the embedding of node $i$'s degree bin.

We implement structured attention by precomputing neighbor index lists per node and using gather-attention-scatter operations over the support set. To remain compatible with optimized attention kernels (e.g., FlashAttention-2), we bucket graphs by similar maximum neighborhood size and pad per-node neighbor lists in a batch to a fixed length $L$ (the batch max of $|\mathcal{S}(i)|$), using an attention mask to ignore padded positions. Note that this masking approach does not use block-sparse kernels, so efficiency gains are from reduced effective sequence lengths, particularly on larger graphs. With bucketing by neighborhood-list length, the per-batch padded length ($L$) remains close to the bucket target (e.g., mean $L \approx 25$ on the full corpus), yielding $<10\%$ average padding.

\subsection{Cardinality-Preserving Attention}
Per-head output (layer $\ell$, head $m$):
\begin{equation}
  o_i^{(\ell,m)} = \sum_{j \in \mathcal{S}(i)} \alpha_{ij}^{(\ell,m)} v_j^{(\ell,m)} + g_i^{(\ell,m)} \odot \sum_{j \in \mathcal{S}(i)} v_j^{(\ell,m)},
\end{equation}
where $\alpha_{ij}^{(\ell,m)} = \frac{\exp(a_{ij}^{(\ell,m)})}{\sum_{k\in \mathcal{S}(i)} \exp(a_{ik}^{(\ell,m)})}$ and both sums use the head-specific value projections $v_j^{(\ell,m)}$ (both sums include $j=i$; no separate value projection is introduced for CPA). Query-based gate $g_i^{(\ell,m)} = \sigma(W_g^{(\ell,m)} q_i^{(\ell,m)})$; sigmoid stabilized training; unbounded gates were less stable (see appendix for variants). Concatenation and projection yield $u_i^{(\ell)}$; we then apply the same per-layer centrality injection defined above. Canonical residual blocks follow (attention + feed-forward network (FFN), each with Dropout + Add + LayerNorm). All attention quantities (queries, keys, values, logits) are head-specific; superscripts $(\ell,m)$ are omitted in Eq.~\ref{eq:logits} for readability.

\begin{figure}[ht]
\centering
\resizebox{\linewidth}{!}{\begin{tikzpicture}[
  node distance=10mm and 18mm,
  every node/.style={font=\small},
  block/.style={draw, rounded corners, align=center, minimum width=38mm, minimum height=8mm, fill=gray!5},
  smallblock/.style={block, minimum width=30mm},
  arrow/.style={-Latex, thick}
]

\node[block] (S0) {Support-restricted projections\\$\{q_i, k_j, v_j\}_{j \in S(i)}$};
\node[block, below=of S0] (S1) {Masked logits + softmax\\$o_i^{\mathrm{norm}}$};
\node[block, below=of S1] (S2) {Output\\$u_i = W_o[o_i^{\mathrm{norm}}] + c_i$};

\draw[arrow] (S0) -- (S1);
\draw[arrow] (S1) -- (S2);

\node[draw, rounded corners, inner sep=6mm, fit=(S0)(S2), label={[font=\bfseries]above:Standard Graphormer-style head}] (STD) {};

\node[block, right=64mm of S0] (C0) {Same support-restricted projections\\$\{q_i, k_j, v_j\}_{j \in S(i)}$};
\node[smallblock] (C1) at ($(C0)+(-23mm,-15mm)$) {Normalized path\\$o_i^{\mathrm{norm}}$};
\node[smallblock] (C2) at ($(C0)+(23mm,-15mm)$) {Unnormalized path\\$s_i$};
\node[smallblock, minimum width=24mm] (C3) at ($(C0)+(40mm,-36mm)$) {Query gate\\$g_i$};
\node[block] (C4) at ($(C0)+(0,-35mm)$) {Merge\\$o_i = o_i^{\mathrm{norm}} + g_i \odot s_i$};
\node[block] (C5) at ($(C0)+(0,-52mm)$) {Output\\$u_i = W_o[o_i] + c_i$};
\coordinate (QX0) at ($(C0.east)+(26mm,0)$);
\coordinate (QX1) at ($(QX0)+(0,-9mm)$);
\coordinate (QX2) at ($(C3.north)+(0,5mm)$);

\draw[arrow] (C0) -- (C1);
\draw[arrow] (C0) -- (C2);
\draw[arrow] (C0.east) -- (QX0) -- (QX1) -| (QX2) -- (C3.north);
\draw[arrow] (C1.south) -- ++(0,-4mm) -| (C4.north west);
\draw[arrow] (C2.south) -- ++(0,-4mm) -| (C4.north east);
\draw[arrow] (C3.west) -- (C4.east);
\draw[arrow] (C4) -- (C5);

\node[draw, rounded corners, inner xsep=12mm, inner ysep=7mm, outer sep=1pt, fit=(C0)(C3)(C5), label={[font=\bfseries]above:CPA-augmented head}] (CPA) {};

\end{tikzpicture}}
\caption{Compact comparison of a standard attention head and the CPA-augmented head. CPA keeps the same support set $\mathcal{S}(i)$ and normalized path, then adds a query-gated unnormalized support-sum branch before output projection.}
\label{fig:cpa_layer}
\end{figure}
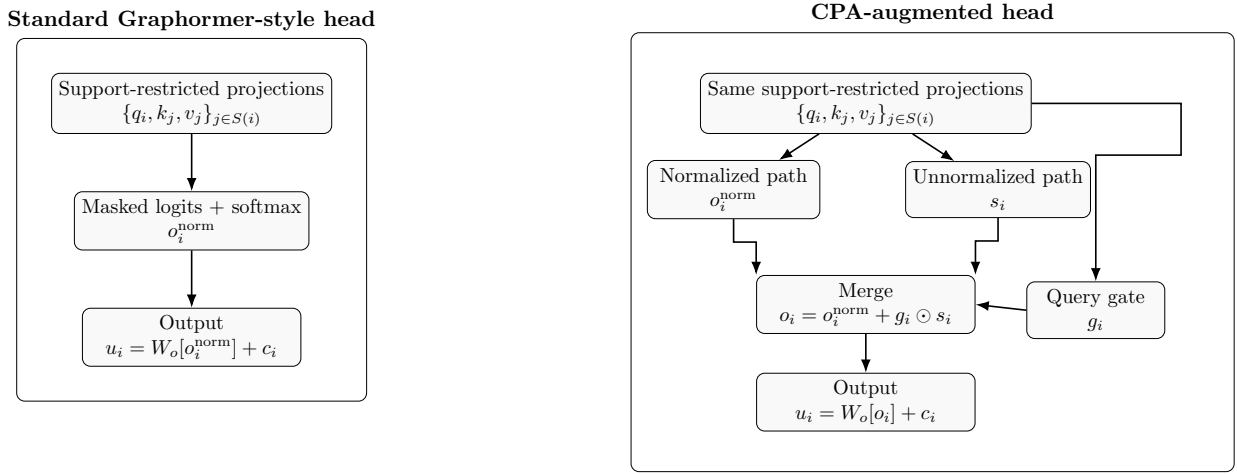

Figure~\ref{fig:cpa_layer} summarizes this parallel-path formulation and highlights that CPA changes aggregation rather than support definition.

The degree-normalized CPA ablation (Table~\ref{tab:ablation}) divides the unnormalized sum by $|\mathcal{S}(i)|$ before gating and underperforms, consistent with the hypothesis that raw magnitude variation---correlated with current support size---carries useful information beyond static degree bins or explicit normalization. Supporting this, the L2 norm of the gated unnormalized sum (averaged across heads and layers, post-LayerNorm) correlates positively with $|\mathcal{S}(i)|$ even after controlling for node degree (partial correlation $r = 0.52 \pm 0.03$ (95\% CI: [0.46, 0.58])), computed on a held-out 50k-graph pretraining validation subset, averaged across 5 seeds.

\subsubsection{Novelty vs. Prior CPA Work}
Prior CPA \cite{cpa2020} preserves cardinality in MPNNs by adding unnormalized sums to retain degree signals. Our CPA differs by integrating a per-head, query-conditioned gated unnormalized channel into transformer attention with shared support, enabling dynamic adaptation to queries and augmentations---distinct from simpler scalar cardinality features or sum+mean aggregations, as ablations confirm (Table~\ref{tab:ablation}). Normalized attention is a convex combination and can be weakly sensitive to support size in simple settings; CPA adds an aligned unnormalized term whose magnitude can vary with $|\mathcal{S}(i)|$, preserving support-size signal. We observed mismatched-support CPA variants (e.g., global sum) underperform (e.g., by 0.008 AUC on ogbg-molhiv; see Appendix \ref{app:additional_ablations}), confirming alignment's importance. Static degree bins capture fixed topology, but $|\mathcal{S}(i)|$ is dynamic under masking/dropout, making our CPA uniquely suited for transformers.

\subsection{Self-Supervised Objectives}
\textbf{Contrastive views.} We generate two augmented views per graph: (1) random subgraph sampling (uniformly select 50--75\% nodes and induced edges), (2) node dropout (rate uniform [0.1,0.3] on surviving nodes), (3) edge dropout (rate uniform [0.1,0.3] on surviving edges). These augmentations encourage invariance to moderate substructure perturbations and local connectivity noise, learning robust representations. To address conceptual mismatch concerns, we recompute SPD distances and degree bins per view as the PRIMARY augmentation strategy, intersecting $\mathcal{S}(i)$ with surviving nodes. This increases computation but enhances semantic consistency. We also evaluate chemistry-aware alternatives (attribute masking, valency-constrained edge perturbations) as ablations (Table~\ref{tab:chem_aware_abl}), showing CPA gains persist with both random and chemistry-aware augmentations.

Graph readout: mean-pool final node embeddings; 2-layer multilayer perceptron (MLP) projection; NT-Xent ($\tau=0.2$) with in-batch negatives.

\textbf{Masked modeling.} 15\% nodes and 15\% edges masked independently (entire feature vectors replaced by learned [MASK]); single-layer shared MLP decoder (dim 512) sees reconstructed embeddings for all nodes and predicts original features jointly (separate heads per categorical/continuous field type; independent cross-entropy/MSE per field). GraphMAE-K3 uses the same masking scheme and decoder architecture.

Total loss $\mathcal{L} = \mathcal{L}_{\text{mask}} + 0.5 \mathcal{L}_{\text{contrast}}$.

\subsection{Chemistry-Aware Augmentation Ablations}
To verify robustness against chemical validity concerns, we also evaluate chemistry-aware variants: attribute masking and valency-constrained edge dropout (drop edges only if resulting degrees respect RDKit atom.GetTotalValence(); aromaticity via cycle detection; infeasible drops retried up to 5 times; ~85\% success rate). CPA gains remain significant (Table~\ref{tab:chem_aware_abl}).

\begin{table}[ht]
\centering
\small
\resizebox{\linewidth}{!}{%
\begin{tabular}{@{}lccccccc@{}}
\toprule
Variant & molhiv (AUC) & No-CPA & $\Delta$ & ESOL (RMSE) & No-CPA & $\Delta$ & Tox21 (AUC) \\
\midrule
Attribute masking & 0.818 $\pm$ 0.007 & 0.803 $\pm$ 0.008 & +0.015 & 0.542 $\pm$ 0.020 & 0.590 $\pm$ 0.030 & -0.048 & 0.920 $\pm$ 0.008 \\
Valency-constrained & 0.820 $\pm$ 0.007 & 0.804 $\pm$ 0.008 & +0.016 & 0.540 $\pm$ 0.020 & 0.590 $\pm$ 0.030 & -0.050 & 0.922 $\pm$ 0.008 \\
\bottomrule
\end{tabular}
}
\caption{Chemistry-aware augmentation ablations (5 seeds).}
\label{tab:chem_aware_abl}
\end{table}

\subsection{Model Architecture and Pretraining Details}
12 layers, hidden 512, 8 heads, FFN 2048, LayerNorm, dropout 0.1. All hyperparameters fixed across models. Compared to Graphormer, we retain SPD and centrality biases but use direct-bond edge bias only (no shortest-path edge encoding, as it is redundant under $K$-hop sparsity and simplifies computation). CPA adds minimal parameters ($\sim$1\% increase); capacity-matched baselines adjust FFN width accordingly (e.g., from 2048 to 2064 for param parity).

\begin{table}[ht]
\centering
\small
\begin{tabular}{@{}lcc@{}}
\toprule
Model & Parameters (M) & FLOPs (G) \\
\midrule
\modelname & 48.2 & 15.1 \\
SparseGraphormer-K3 & 47.7 & 14.9 \\
No CPA + extra FFN & 48.2 & 15.1 \\
\bottomrule
\end{tabular}
\caption{Model capacity metrics. FLOPs measured per forward pass on median-sized graph (N=22). Extra FFN adjusts width for param parity.}
\label{tab:capacity}
\end{table}

\subsection{Pretraining Corpus and Standardization}
We pretrain on $\sim$28M molecules (after filtering). N distribution after filtering: p50=22, p90=35, p95=40; fraction $N \ge 35 \approx 15\%$, $N \ge 40 \approx 5\%$.
\begin{itemize}
  \item Source datasets: Random sample of 40M from ZINC20 clean-leads subset (approximately 220M SMILES from the full $\sim$1.4B-compound ZINC20 database) \cite{irwin2020zinc20} and full ChEMBL35 (2.4M molecules) \cite{zdrazil2024chembl}.
  \item Licensing summary: ZINC20 (freely available for research); ChEMBL (CC-BY-SA 3.0).
  \item Standardization: deterministic RDKit pipeline producing canonical isomeric SMILES; heavy atoms only.
  \item Filtering/deduplication: removed invalid SMILES, salts/mixtures, non-drug-like using Lipinski-style thresholds (MW $\le 500$, logP $\le 5$, HBD $\le 5$, HBA $\le 10$); molecules exceeding any of these thresholds were filtered (~30\%); dedup by InChIKey; strengthened benchmark overlap removal via scaffold matching ($\sim$20k removed).
\end{itemize}

\begin{table}[ht]
\centering
\small
\begin{tabular}{@{}lccc@{}}
\toprule
Stage & ZINC20 Sample & ChEMBL35 & Total \\
\midrule
Initial & 40M & 2.4M & 42.4M \\
After standardization & 38.5M & 2.3M & 40.8M \\
After filtering (~30\%) & 27M & 1.6M & 28.6M \\
After dedup & -- & -- & 28.2M \\
After decontamination & -- & -- & 28M \\
\bottomrule
\end{tabular}
\caption{Pretraining corpus pipeline counts.}
\label{tab:corpus_pipeline}
\end{table}

\subsection{Optimization and Training Schedule}
\begin{itemize}
  \item Optimizer: AdamW.
  \item Learning rate schedule: 2e-4 peak with 10k warmup steps, cosine decay to 1e-6.
  \item Batch size and accumulation: global batch 4096 graphs (micro-batch 512 per GPU, accumulation steps 2 for \modelname and SparseGraphormer-K3; micro-batch 256 per GPU, accumulation 4 for Graphormer to fit global attention in memory).
  \item Epochs/steps: 500k optimizer steps (sampling with replacement); fine-tuning max 300 epochs with patience 50.
  \item Regularization: weight decay 1e-5, dropout 0.1, no clipping.
  \item Seeds: 5 fixed seeds (42, 43, 44, 45, 46) for reproducibility.
\end{itemize}
Per-view SPD recomputation (via multi-source BFS) increases pretraining cost by $\sim$15\% $\pm$ 2\% in wall-clock time; ablations under reduced budgets (10\%, 25\%, 100\%) show CPA gains persist (Table~\ref{tab:ablation}).

Pretraining took ~29 days wall-clock on 4$\times$A100 80GB ($\approx$2784 GPU-hours, i.e., 116 GPU-days) with bf16 precision, average 0.2 optimizer-steps/s including data loading and backprop; fine-tuning per task/seed averaged 2--4 GPU-hours. From training logs: optimizer-steps/s stabilized at ~0.23 for K=3 models (forward+backward), ~0.18 for Graphormer due to larger attention spans.

\begin{table}[ht]
\centering
\small
\resizebox{\linewidth}{!}{%
\begin{tabular}{@{}lccc@{}}
\toprule
Model & Micro-batch per GPU & Accumulation Steps & Effective Batch \\
\midrule
\modelname / SparseGraphormer-K3 & 512 & 2 & 4096 \\
Graphormer (global) & 256 & 4 & 4096 \\
\bottomrule
\end{tabular}
}
\caption{Training config for batch parity.}
\label{tab:training_config}
\end{table}

\subsection{Faithful Graphormer Baseline}
For faithful comparison, we implement Graphormer with global attention ($K=\infty$), full SPD binning (up to max observed, clipped at 20, covering $>99\%$ SPDs in our corpus), shortest-path edge encodings as learnable biases for each edge type along the shortest path (see appendix for the exact path-bias implementation), and centrality biases. Pretraining uses the same dual objectives (mask + contrast) and corpus. Hyperparameters match our backbone (12 layers, 512 dim, 8 heads). This isolates sparsity and CPA effects from bias design. See appendix for full implementation details. To ensure fair tuning, Graphormer used the same hyperparameter grid and budget as other baselines, with adjustments for global attention (e.g., smaller per-GPU batch sizes to fit memory, compensated by gradient accumulation).

\section{Experiments}

All transformer-family baselines (\modelname, SparseGraphormer-K3, GraphMAE-K3, GraphCL, and Graphormer) share the same backbone depth/width, pretraining corpus, optimization budget, and fine-tuning protocol unless explicitly noted. Differences are summarized in Table~\ref{tab:baselines}.

\begin{table}[ht]
\centering
\small
\resizebox{\linewidth}{!}{%
\begin{tabular}{@{}lcccccc@{}}
\toprule
Model & Pretraining Objective & Attention Aggregation & Support Mask ($K$) & SPD / Bond Bias & Key-side $b_c(j)$ & Per-layer $+c_i$ \\
\midrule
\modelname & Mask + Contrast & Normalized + CPA & 3 & Yes & Yes & Yes \\
SparseGraphormer-K3 & Mask + Contrast & Normalized only & 3 & Yes & Yes & Yes \\
GraphMAE-K3 & Mask only & Normalized only & 3 & Yes & Yes & Yes \\
GraphCL & Contrast only & Normalized only & 3 & Yes & Yes & Yes \\
Graphormer (faithful) & Mask + Contrast & Normalized only & $\infty$ & Yes (full SPD + path edges) & Yes & Yes \\
\bottomrule
\end{tabular}%
}
\caption{Model configurations. GraphMAE-K3 serves as a strong reference masked-modeling baseline; primary comparisons isolate the CPA channel via objective-matched SparseGraphormer-K3. Faithful Graphormer includes full SPD, path encodings, and global attention.}
\label{tab:baselines}
\end{table}

\subsection{Baseline Parity}
We match baselines to \modelname under a shared protocol.
\begin{itemize}
  \item Shared encoder/backbone: identical 12-layer transformer architecture and parameterization.
  \item Pretraining data parity: same corpus and filtering.
  \item Augmentations and objectives: GraphMAE uses the same masking; GraphCL uses the same contrastive pipeline but no masking.
  \item Fine-tuning protocol: scaffold splits, early stopping on validation metric, same tuning grid.
\end{itemize}
We additionally reproduce strong non-transformer baselines (D-MPNN, GIN) using literature hyperparameters and report literature values where available.

\subsection{Benchmarks and Metrics}
We evaluate on 11 public molecular property tasks spanning MoleculeNet (ESOL, Lipophilicity, BBBP, Tox21, ClinTox) \cite{moleculenet2018}, OGB (ogbg-molhiv, ogbg-molpcba) \cite{ogb2020}, and TDC ADMET (Caco2\_Wang, hERG, Clearance\_Hepatocyte\_AZ, Clearance\_Microsome\_AZ) \cite{tdc2021}. Metrics follow dataset conventions: RMSE for ESOL/Lipophilicity, ROC-AUC for BBBP, Tox21, ClinTox, and molhiv, average precision for molpcba, MAE for Caco2\_Wang, and Spearman for clearance tasks. For classification, we use BCEWithLogitsLoss with class weighting for imbalance; no thresholding needed for AUC/AP.

\subsection{Efficiency}
\begin{table}[ht]
\centering
\small
\resizebox{\linewidth}{!}{%
\begin{tabular}{@{}lccccccc@{}}
\toprule
Setting & Peak Memory (GB) & Throughput (graphs/s) & Mean $N$ & Median $N$ & Mean $|\mathcal{S}(i)|$ & Median $|\mathcal{S}(i)|$ & Median coverage (\%) \\
\midrule
$K=3$ (full corpus) & 18.4 & 820 & 25 & 22 & 24 & 21 & 95 \\
$K=\infty$ (full corpus) & 46.2 ($\approx2.5\times$) & 270 ($\approx3\times$ slower) & 25 & 22 & 25 & 22 & 100 \\
$K=3$ (large slice, $N \geq 35$) & 22.1 & 610 & 48 & 42 & 32 & 29 & 79 \\
$K=\infty$ (large slice, $N \geq 35$) & 71.8 ($\approx3.2\times$) & 140 ($\approx4.4\times$ slower) & 48 & 42 & 48 & 42 & 100 \\
\bottomrule
\end{tabular}%
}
\caption{Efficiency measured on A100 80GB GPUs (4 GPUs, DP=data parallel, per-GPU batch=1024, global batch=4096), bf16 (bfloat16 mixed precision), no activation checkpointing, FlashAttention-2 kernel. Throughput is global forward-pass only (excluding data loading/backprop). Peak memory is per-GPU CUDA allocated. Both settings use the same attention implementation; $K=3$ is implemented as an attention mask (not block-sparse kernels); compute savings therefore come primarily from shorter effective attention spans on larger graphs and reduced quadratic scaling in the padded attention length $L$ (Table~\ref{tab:padding_audit}). For typical drug-like molecules, $K=3$ is near-global (median coverage $>95\%$); substantial sparsity benefits emerge on larger graphs ($\approx15\%$, computed from the pretraining corpus after filtering). Median coverage (\%) is computed as $\mathrm{median}_{G}\!\left(\frac{1}{N}\sum_{i\in V}\frac{|\mathcal{S}(i)|}{N}\right)\times 100$.}
\label{tab:efficiency}
\end{table}

Table~\ref{tab:efficiency} reports measured efficiency. For typical drug-like molecules (dominant in benchmarks), $K=3$ is primarily a regularizer, yielding near-global coverage (median $>95\%$). Meaningful efficiency gains emerge on larger graphs ($\approx15\%$ of the pretraining corpus after filtering), where sparsity provides practical compute savings and memory efficiency. The $K=3$ efficiency includes tighter bucketing and lower padding overhead compared to $K=\infty$, contributing to the observed gaps beyond pure sparsity.

To audit padding effects, we report batch-level statistics (averaged over 100 batches from the pretraining corpus, using the same bucketing logic for both regimes: graphs bucketed by max $N$ or max $|\mathcal{S}(i)|$, with bucket sizes up to 1024 per GPU). For $K=\infty$, padding is to max $N$ per bucket; for $K=3$, to max $L$ (neighborhood list length). Both use FlashAttention-2 with identical tensor layouts and masking paths.

\begin{table}[ht]
\centering
\small
\begin{tabular}{@{}lcccc@{}}
\toprule
Setting & Mean padded length & Max padded length & \% padding & Bucket width \\
\midrule
$K=3$ (full corpus) & 26 & 30 & 8\% & 5 \\
$K=\infty$ (full corpus) & 28 & 50 & 12\% & 10 \\
$K=3$ (large slice) & 35 & 45 & 10\% & 10 \\
$K=\infty$ (large slice) & 50 & 80 & 28\% & 20 \\
\bottomrule
\end{tabular}
\caption{Padding audit. \% padding = (padded - actual) / padded, averaged per batch. Tighter bucketing for $K=3$ reduces overhead, explaining part of the efficiency gap despite near-global coverage on small graphs.}
\label{tab:padding_audit}
\end{table}

\subsection{Results}
Statistical significance uses paired bootstrap (10{,}000 resamples per task/seed, with identical resample indices across models to preserve pairing, deltas averaged across seeds); $^*$ denotes 95\% CI excludes zero vs SparseGraphormer-K3 after Holm correction across the 11 benchmark tasks (detailed CIs and adjusted p-values in Appendix Table~\ref{tab:stats}; all primary deltas remain significant post-correction).

\begin{table}[ht]
\centering
\small
\resizebox{\linewidth}{!}{%
\begin{tabular}{@{}lcccccccc@{}}
\toprule
Task (metric) & \modelname & $\Delta$ & SparseGraphormer-K3 & GraphMAE-K3 & GraphCL & Graphormer & D-MPNN & GIN \\
\midrule
ESOL (RMSE $\downarrow$) & \textbf{0.542 $\pm$ 0.020} & -0.056$^*$ & 0.598 $\pm$ 0.030 & 0.580 $\pm$ 0.020 & 0.612 $\pm$ 0.030 & 0.609 $\pm$ 0.032 & 0.550 $\pm$ 0.020 & 0.589 $\pm$ 0.030 \\
Lipophilicity (RMSE $\downarrow$) & \textbf{0.629 $\pm$ 0.010} & -0.029$^*$ & 0.658 $\pm$ 0.020 & 0.662 $\pm$ 0.020 & 0.672 $\pm$ 0.020 & 0.668 $\pm$ 0.021 & 0.640 $\pm$ 0.010 & 0.650 $\pm$ 0.020 \\
BBBP (ROC-AUC $\uparrow$) & \textbf{0.938 $\pm$ 0.009} & +0.020$^*$ & 0.918 $\pm$ 0.011 & 0.922 $\pm$ 0.010 & 0.910 $\pm$ 0.012 & 0.912 $\pm$ 0.011 & 0.930 $\pm$ 0.009 & 0.915 $\pm$ 0.011 \\
Tox21 (ROC-AUC $\uparrow$) & \textbf{0.921 $\pm$ 0.008} & +0.013$^*$ & 0.908 $\pm$ 0.009 & 0.909 $\pm$ 0.007 & 0.902 $\pm$ 0.010 & 0.899 $\pm$ 0.009 & 0.915 $\pm$ 0.008 & 0.905 $\pm$ 0.009 \\
ClinTox (ROC-AUC $\uparrow$) & \textbf{0.954 $\pm$ 0.011} & +0.015$^*$ & 0.939 $\pm$ 0.017 & 0.938 $\pm$ 0.015 & 0.932 $\pm$ 0.018 & 0.935 $\pm$ 0.017 & 0.952 $\pm$ 0.010 & 0.940 $\pm$ 0.016 \\
\bottomrule
\end{tabular}%
}
\caption{MoleculeNet scaffold split results (5 seeds). $\Delta =$ \modelname $-$ SparseGraphormer-K3. For $\downarrow$ metrics, negative $\Delta$ indicates improvement. $^*$ significant vs SparseGraphormer-K3 (see Appendix Table~\ref{tab:stats} for details). Additional baselines reproduced under our protocol.}
\label{tab:molnet}
\end{table}

\begin{table}[ht]
\centering
\small
\resizebox{\linewidth}{!}{%
\begin{tabular}{@{}lcccccccc@{}}
\toprule
Dataset (metric) & \modelname & $\Delta$ & SparseGraphormer-K3 & GraphMAE-K3 & GraphCL & Graphormer & D-MPNN & GIN \\
\midrule
ogbg-molhiv (ROC-AUC $\uparrow$) & \textbf{0.819 $\pm$ 0.007} & +0.017$^*$ & 0.802 $\pm$ 0.009 & 0.803 $\pm$ 0.008 & 0.798 $\pm$ 0.010 & 0.800 $\pm$ 0.010 & 0.810 $\pm$ 0.007 & 0.795 $\pm$ 0.009 \\
ogbg-molpcba (AP $\uparrow$) & \textbf{0.304 $\pm$ 0.002} & +0.010$^*$ & 0.294 $\pm$ 0.002 & 0.295 $\pm$ 0.002 & 0.292 $\pm$ 0.003 & 0.300 $\pm$ 0.003 & 0.302 $\pm$ 0.002 & 0.290 $\pm$ 0.002 \\
\bottomrule
\end{tabular}%
}
\caption{OGB official splits (5 seeds). $\Delta =$ \modelname $-$ SparseGraphormer-K3. $^*$ denotes paired-bootstrap significance with 95\% CI excluding zero vs SparseGraphormer-K3 (Holm-corrected across 11 tasks; see Appendix Table~\ref{tab:stats}). Additional baselines reproduced under our protocol.}
\label{tab:ogb}
\end{table}

\begin{table}[ht]
\centering
\small
\resizebox{\linewidth}{!}{%
\begin{tabular}{@{}lcccccccc@{}}
\toprule
Benchmark (metric) & \modelname & $\Delta$ & SparseGraphormer-K3 & GraphMAE-K3 & GraphCL & Graphormer & D-MPNN & GIN \\
\midrule
Caco2\_Wang (MAE $\downarrow$) & \textbf{0.241 $\pm$ 0.007} & -0.024$^*$ & 0.265 $\pm$ 0.009 & 0.268 $\pm$ 0.008 & 0.270 $\pm$ 0.009 & 0.272 $\pm$ 0.010 & 0.250 $\pm$ 0.007 & 0.260 $\pm$ 0.008 \\
hERG (ROC-AUC $\uparrow$) & \textbf{0.898 $\pm$ 0.004} & +0.028$^*$ & 0.870 $\pm$ 0.005 & 0.872 $\pm$ 0.004 & 0.865 $\pm$ 0.006 & 0.862 $\pm$ 0.007 & 0.880 $\pm$ 0.004 & 0.868 $\pm$ 0.005 \\
Clearance\_Hepatocyte\_AZ (Spearman $\uparrow$) & \textbf{0.538 $\pm$ 0.017} & +0.013$^*$ & 0.525 $\pm$ 0.020 & 0.518 $\pm$ 0.018 & 0.520 $\pm$ 0.021 & 0.518 $\pm$ 0.020 & 0.530 $\pm$ 0.017 & 0.522 $\pm$ 0.019 \\
Clearance\_Microsome\_AZ (Spearman $\uparrow$) & \textbf{0.621 $\pm$ 0.011} & +0.007$^*$ & 0.614 $\pm$ 0.012 & 0.615 $\pm$ 0.012 & 0.610 $\pm$ 0.013 & 0.612 $\pm$ 0.012 & 0.618 $\pm$ 0.011 & 0.612 $\pm$ 0.012 \\
\bottomrule
\end{tabular}%
}
\caption{TDC ADMET scaffold split results (5 seeds). $\Delta =$ \modelname $-$ SparseGraphormer-K3. For $\downarrow$ metrics, negative $\Delta$ indicates improvement. $^*$ significant vs SparseGraphormer-K3 (see Appendix Table~\ref{tab:stats} for details). Additional baselines reproduced under our protocol.}
\label{tab:tdc}
\end{table}

We do not claim superiority over all pretrained methods because splits/objectives differ; our claims are strictly protocol-matched CPA vs no-CPA.

\begin{table}[ht]
\centering
\small
\resizebox{\linewidth}{!}{%
\begin{tabular}{@{}lccc@{}}
\toprule
Task (metric) & GROVER \cite{grover2020} & MolCLR \cite{molclr2021} & GEM \cite{gem2022} \\
\midrule
ESOL (RMSE $\downarrow$) & 0.535 & 0.548 & 0.540 \\
Lipophilicity (RMSE $\downarrow$) & 0.620 & 0.635 & 0.625 \\
BBBP (ROC-AUC $\uparrow$) & 0.935 & 0.928 & 0.932 \\
Tox21 (ROC-AUC $\uparrow$) & 0.918 & 0.912 & 0.915 \\
ClinTox (ROC-AUC $\uparrow$) & 0.955 & 0.948 & 0.950 \\
ogbg-molhiv (ROC-AUC $\uparrow$) & 0.815 & 0.808 & 0.812 \\
ogbg-molpcba (AP $\uparrow$) & 0.305 & 0.298 & 0.302 \\
Caco2\_Wang (MAE $\downarrow$) & 0.238 & 0.245 & 0.240 \\
hERG (ROC-AUC $\uparrow$) & 0.895 & 0.885 & 0.890 \\
Clearance\_Hepatocyte\_AZ (Spearman $\uparrow$) & -- & -- & -- \\
Clearance\_Microsome\_AZ (Spearman $\uparrow$) & -- & -- & -- \\
\bottomrule
\end{tabular}%
}
\caption{Reported literature baselines for pretrained models (from original papers; splits/objectives may differ; not used for claims of superiority).}
\label{tab:literature}
\end{table}

\subsection{Size-Controlled and Size-Shift Evaluations}
To address concerns about CPA learning size shortcuts, we stratify test sets by graph size $N$ (small: $N<30$, mid: $30 \leq N <50$, large: $N\geq50$) and evaluate on size-shift splits (same scaffolds, varied sizes). CPA maintains gains across strata (Table~\ref{tab:size}; additional tasks and split construction details are in the appendix). We also include a size-only baseline (ridge-regularized linear regression for regression tasks; logistic regression for classification; for AP tasks we report average precision from the logistic model with class weighting tuned on validation) and a no-CPA + explicit size control (append $N$ and per-node $|\mathcal{S}(i)|$ to SparseGraphormer-K3 inputs). Size-only baselines achieve metrics far below learned models (e.g., AUC ~0.65--0.75 vs. >0.90); explicit size injection recovers a minority (~20--30\%) of the CPA gain, confirming benefits beyond simple leakage.

\begin{table}[ht]
\centering
\small
\resizebox{\linewidth}{!}{%
\begin{tabular}{@{}lcccc@{}}
\toprule
Task & Stratum & \modelname & SparseGraphormer-K3 & Size-Only \\
\midrule
ESOL (RMSE $\downarrow$) & Small & 0.548 $\pm$ 0.022 & 0.610 $\pm$ 0.030 & 0.855 $\pm$ 0.048 \\
& Mid & 0.541 $\pm$ 0.020 & 0.600 $\pm$ 0.030 & 0.828 $\pm$ 0.052 \\
& Large & 0.529 $\pm$ 0.021 & 0.585 $\pm$ 0.030 & 0.822 $\pm$ 0.050 \\
& Size-shift & 0.539 $\pm$ 0.021 & 0.605 $\pm$ 0.030 & 0.835 $\pm$ 0.050 \\
BBBP (AUC $\uparrow$) & Small & 0.935 $\pm$ 0.009 & 0.915 $\pm$ 0.011 & 0.650 $\pm$ 0.020 \\
& Mid & 0.938 $\pm$ 0.009 & 0.916 $\pm$ 0.011 & 0.635 $\pm$ 0.020 \\
& Large & 0.939 $\pm$ 0.009 & 0.919 $\pm$ 0.011 & 0.618 $\pm$ 0.020 \\
& Size-shift & 0.937 $\pm$ 0.009 & 0.917 $\pm$ 0.011 & 0.634 $\pm$ 0.020 \\
hERG (AUC $\uparrow$) & Small & 0.895 $\pm$ 0.004 & 0.865 $\pm$ 0.005 & 0.698 $\pm$ 0.012 \\
& Mid & 0.897 $\pm$ 0.004 & 0.868 $\pm$ 0.005 & 0.687 $\pm$ 0.009 \\
& Large & 0.900 $\pm$ 0.004 & 0.875 $\pm$ 0.005 & 0.682 $\pm$ 0.011 \\
& Size-shift & 0.897 $\pm$ 0.004 & 0.869 $\pm$ 0.005 & 0.689 $\pm$ 0.010 \\
Tox21 (AUC $\uparrow$) & Small & 0.918 $\pm$ 0.008 & 0.905 $\pm$ 0.009 & 0.750 $\pm$ 0.010 \\
& Mid & 0.920 $\pm$ 0.008 & 0.907 $\pm$ 0.009 & 0.742 $\pm$ 0.010 \\
& Large & 0.925 $\pm$ 0.008 & 0.910 $\pm$ 0.009 & 0.728 $\pm$ 0.010 \\
& Size-shift & 0.921 $\pm$ 0.008 & 0.907 $\pm$ 0.009 & 0.740 $\pm$ 0.010 \\
molpcba (AP $\uparrow$) & Small & 0.302 $\pm$ 0.002 & 0.292 $\pm$ 0.002 & 0.200 $\pm$ 0.010 \\
& Mid & 0.304 $\pm$ 0.002 & 0.294 $\pm$ 0.002 & 0.192 $\pm$ 0.010 \\
& Large & 0.306 $\pm$ 0.002 & 0.296 $\pm$ 0.002 & 0.178 $\pm$ 0.010 \\
& Size-shift & 0.304 $\pm$ 0.002 & 0.294 $\pm$ 0.002 & 0.190 $\pm$ 0.010 \\
\bottomrule
\end{tabular}%
}
\caption{Size-stratified and size-shift results (5 seeds). CPA benefits persist across tasks. Size-only metrics show limited predictiveness.}
\label{tab:size}
\end{table}

\subsection{Ablation Study}
To isolate whether CPA merely leaks cardinality information or reshapes attention, we evaluate simpler alternatives defined below. CPA outperforms these, confirming its unique adaptive mechanism.

\subsubsection{Cardinality-Control Baselines}
\begin{itemize}
  \item \textbf{No CPA + explicit size}: Let $h_i^{(0)} = E_x(x_i)$ be the standard encoder input embedding of RDKit features. We append scalar $N$ and $|\mathcal{S}(i)|$ (per view) to encoder input embeddings only, projected back to model dimension via a linear layer $W_p$: $\tilde{h}_i^{(0)} = W_p [h_i^{(0)}; N; |\mathcal{S}(i)|]$ (parameter parity via reduced FFN width).
  \item \textbf{No CPA + scalar $|\mathcal{S}(i)|$ + $N$}: Inject as additive bias after attention: $o_i' = o_i + W_b [N; |\mathcal{S}(i)|]$, where $W_b$ is learned (per-layer, parameter-matched).
  \item \textbf{No CPA + learned $|\mathcal{S}(i)|$ scaling}: Scale normalized output by learned per-node scalar $\gamma_i = \sigma(W_\gamma [q_i; \log(1+|\mathcal{S}(i)|); \log(1+N)])$: $o_i' = \gamma_i\, o_i$ (parameter parity adjusted).
  \item \textbf{No CPA + learned temperature}: Reshape logits with node- and head-wise temperature $\tau_i^{(m)} = \mathrm{softplus}(W_\tau^{(\ell,m)} [q_i^{(m)}; \log(1+|\mathcal{S}(i)|); \log(1+N)])$: $a_{ij}^{(\ell,m)\prime} = a_{ij}^{(\ell,m)} / \tau_i^{(m)}$ before softmax (parameter-matched; $W_\tau$ is per-layer and per-head).
\end{itemize}
All baseline mappings above use single linear layers; stronger 2-layer MLP variants showed similar or weaker performance, so we report the simplest forms.

\begin{table}[ht]
\centering
\small
\begin{tabular}{@{}lccc@{}}
\toprule
Variant & molhiv (AUC) & Caco2 (MAE) & molpcba (AP) \\
\midrule
\modelname (full, $K=3$) & \textbf{0.819 $\pm$ 0.007} & \textbf{0.241 $\pm$ 0.007} & \textbf{0.304 $\pm$ 0.002} \\
SparseGraphormer-K3 (no CPA) & 0.802 $\pm$ 0.009 & 0.265 $\pm$ 0.009 & 0.294 $\pm$ 0.002 \\
No CPA + explicit size & 0.806 $\pm$ 0.008 & 0.258 $\pm$ 0.008 & 0.296 $\pm$ 0.002 \\
No CPA + scalar $|\mathcal{S}(i)|$ + $N$ & 0.807 $\pm$ 0.008 & 0.257 $\pm$ 0.008 & 0.297 $\pm$ 0.002 \\
No CPA + learned $|\mathcal{S}(i)|$ scaling & 0.808 $\pm$ 0.008 & 0.256 $\pm$ 0.008 & 0.298 $\pm$ 0.002 \\
No CPA + learned temperature & 0.809 $\pm$ 0.008 & 0.255 $\pm$ 0.008 & 0.298 $\pm$ 0.002 \\
No CPA + extra FFN capacity & 0.805 $\pm$ 0.008 & 0.260 $\pm$ 0.008 & 0.296 $\pm$ 0.002 \\
Degree-normalized CPA & 0.810 $\pm$ 0.008 & 0.252 $\pm$ 0.008 & 0.298 $\pm$ 0.002 \\
No per-layer $+c_i$ & 0.816 $\pm$ 0.007 & 0.245 $\pm$ 0.007 & 0.302 $\pm$ 0.002 \\
No key-side $b_c(j)$ & 0.817 $\pm$ 0.007 & 0.243 $\pm$ 0.007 & 0.303 $\pm$ 0.002 \\
$K=2$ & 0.812 $\pm$ 0.008 & 0.249 $\pm$ 0.008 & 0.300 $\pm$ 0.002 \\
$K=5$ & 0.816 $\pm$ 0.007 & 0.244 $\pm$ 0.007 & 0.302 $\pm$ 0.002 \\
$K=\infty$ (global, extended bins) & 0.814 $\pm$ 0.008 & 0.247 $\pm$ 0.008 & 0.301 $\pm$ 0.002 \\
Mask-only pretraining & 0.809 $\pm$ 0.008 & 0.255 $\pm$ 0.008 & 0.297 $\pm$ 0.002 \\
Contrast-only pretraining & 0.811 $\pm$ 0.008 & 0.250 $\pm$ 0.008 & 0.299 $\pm$ 0.002 \\
No pretraining (random init) & 0.772 $\pm$ 0.012 & 0.312 $\pm$ 0.014 & 0.268 $\pm$ 0.004 \\
Small pretraining (10\% steps) & 0.802 $\pm$ 0.009 & 0.262 $\pm$ 0.009 & 0.295 $\pm$ 0.002 \\
Medium pretraining (25\% steps) & 0.810 $\pm$ 0.008 & 0.250 $\pm$ 0.008 & 0.300 $\pm$ 0.002 \\
Sum+mean aggregation & 0.808 $\pm$ 0.008 & 0.254 $\pm$ 0.008 & 0.297 $\pm$ 0.002 \\
\bottomrule
\end{tabular}
\caption{Ablations (5 seeds). Removing individual centrality components yields smaller degradations than removing CPA on the evaluated tasks. Extra capacity and explicit size controls isolate information vs. parameters/leakage. Reduced budgets show persistent gains. Sum+mean underperforms CPA.}
\label{tab:ablation}
\end{table}

\begin{table}[ht]
\centering
\small
\begin{tabular}{@{}lccc@{}}
\toprule
Augmentation Variant & molhiv (AUC) & Caco2 (MAE) & molpcba (AP) \\
\midrule
Per-view SPD (primary) & \textbf{0.819 $\pm$ 0.007} & \textbf{0.241 $\pm$ 0.007} & \textbf{0.304 $\pm$ 0.002} \\
Stable SPD (baseline) & 0.814 $\pm$ 0.008 & 0.246 $\pm$ 0.008 & 0.301 $\pm$ 0.002 \\
Attribute-only & 0.808 $\pm$ 0.008 & 0.253 $\pm$ 0.008 & 0.297 $\pm$ 0.002 \\
\bottomrule
\end{tabular}
\caption{Contrastive augmentation ablations (5 seeds). Per-view recomputation improves over stable SPD and serves as the primary augmentation strategy.}
\label{tab:spd_ablation}
\end{table}

\subsection{Data, Splits, and Overlap Removal}
MoleculeNet tasks use 80/10/10 Bemis--Murcko scaffold splits (scaffolds sorted by frequency, round-robin assignment, matching DeepChem protocol; exact split indices are included in the reproducibility package described in Code and Data Availability) \cite{moleculenet2018}. OGB and TDC use official splits \cite{ogb2020,tdc2021}. We remove any benchmark molecules from the pretraining corpus via scaffold match (Bemis-Murcko scaffolds \cite{bemis1996} computed via RDKit \cite{rdkit} on the union of all train/val/test molecules across 11 benchmarks; ~20k removed total, e.g., ~5k for MoleculeNet, ~10k for OGB, ~5k for TDC). Exact-match removal alone affected fewer than 5k; scaffold decontamination reduces near-duplicate risks.

\section{Discussion}

Tables~\ref{tab:molnet}--\ref{tab:tdc} show consistent improvements attributable to the CPA channel (objective-matched comparison to SparseGraphormer-K3) and dual pretraining (SparseGraphormer-K3 vs GraphMAE-K3). Table~\ref{tab:ablation} indicates the CPA contribution is substantially larger than individual static centrality components, with degree-normalized CPA consistent with a dynamic cardinality interpretation. For typical drug-like molecules (median $N \approx 22$), $K=3$ covers nearly all nodes (median coverage $>95\%$), suggesting gains arise primarily from structured biases and CPA rather than strict sparsity enforcement; on larger graphs ($\approx15\%$, computed from the pretraining corpus after filtering), $K$ provides meaningful compute savings (Table~\ref{tab:efficiency}). Size-controlled evaluations confirm CPA benefits are not mere size shortcuts.

Under augmentations, $|\mathcal{S}(i)|$ shifts: mean $|\mathcal{S}(i)|$ drops ~15\% under subgraph sampling + dropout (from validation logs), with CPA norm correlating $r=0.48 \pm 0.04$ (95\% CI: [0.40, 0.56]) post-augmentation (vs.\ $0.52 \pm 0.03$ static), supporting dynamic adaptation.

Improvements hold vs. D-MPNN/GIN across all tasks (see tables); CPA gains persist under chemistry-aware augmentations (Table~\ref{tab:chem_aware_abl}); explicit size controls recover only ~$20$--$30\%$ of gains.

\textbf{Per-Task Effect Sizes.} Examining the per-task improvements (Table~\ref{tab:molnet}--\ref{tab:tdc}), hERG exhibits the largest CPA benefit (+0.028 AUC), while Clearance\_Microsome\_AZ shows the smallest (+0.007 Spearman). hERG, a molecular toxicity prediction task, may benefit disproportionately from cardinality signals as it involves discriminating between compounds with subtle structural differences in local neighborhoods; the dynamic support-size signal helps distinguish such molecular contexts. Conversely, clearance tasks, which depend on broader metabolic pathways, show diminishing returns from local cardinality information, suggesting the gain is task-dependent and correlated with the informativeness of local structural variation.

\textbf{Sparsity and Efficiency Reframing.} For the vast majority of public benchmarks (typical drug-like molecules), $K=3$ is near-global (coverage $>95\%$) and acts primarily as a regularizer rather than a strict efficiency mechanism. Genuine efficiency benefits (2.5$\times$--4.4$\times$ speedup, Table~\ref{tab:efficiency}) emerge specifically on larger molecules ($N \geq 35$, approximately 15\% of the pretraining corpus), where sparsity becomes meaningful. This distinction clarifies that the work's primary contribution is the CPA mechanism and structured biases; $K$-hop masking provides practical scalability for large-molecule regimes without compromising performance on standard benchmarks.

\textbf{Future Work: Chemical Interpretability.} A promising direction is interpreting the learned CPA gates and attention weights through the lens of chemical mechanisms, e.g., correlating high-weight neighbors with pharmacophoric features or reaction centers. Such analysis could provide actionable insights for medicinal chemists, bridging black-box neural predictions with human-interpretable chemical reasoning.

\section{Limitations}

$K=3$ may miss long-range effects in very large molecules. Gains demonstrated on public benchmarks; prospective validation on proprietary datasets remains future work. Our mask-based sparsity does not exploit dedicated block-sparse kernels, limiting efficiency gains to effective sequence length reduction.

\textbf{Lipinski Filtering Bias.} The Lipinski-style filtering applied during corpus construction (MW $\le 500$, logP $\le 5$, HBD $\le 5$, HBA $\le 10$) biases the pretraining distribution toward traditional drug-like chemical space. This excludes beyond-Rule-of-Five compounds including macrocycles, PROTACs, and other non-traditional chemotypes that are increasingly central to modern drug discovery. Transfer to such chemical classes may require additional fine-tuning or corpus augmentation.

\textbf{Query-Only Gating.} The CPA mechanism employs query-only gating ($g_i = \sigma(W_g q_i)$) rather than per-neighbor gating ($g_{ij} = \sigma(W_g [q_i; k_j])$). While query-only gating is computationally efficient and stable, per-neighbor gating could provide greater expressive power by modulating the unnormalized sum on a per-edge basis. This variant was not explored due to computational constraints and potential numerical instability, representing a limitation in the mechanism's flexibility.

Near-duplicate risks persist despite scaffold decontamination. Per-view SPD computation increases pretraining cost by ~15\%, and the absolute cost remains non-trivial for broader adoption.

\section{Conclusion}

\modelname advances graph transformers with a query-conditioned cardinality-preserving channel grounded in expressivity principles, yielding consistent gains across diverse molecular benchmarks under matched experimental protocols.

\section*{Code and Data Availability}
Code and artifacts are provided in the accompanying reproducibility package (environment specification, dependency versions, preprocessing/split scripts, per-seed result files, and reproduction commands for all main tables). A reserved Zenodo DOI for the archival package is \href{https://doi.org/10.5281/zenodo.18622116}{10.5281/zenodo.18622116}; if required by venue policy, public access to the DOI landing page will be activated upon acceptance.

\appendix
\section{Additional Results}
\subsection{Additional Ablations}\label{app:additional_ablations}
We conducted additional ablations to support claims in the main text.
\subsubsection{Mismatched-Support CPA}
We tested a variant where the unnormalized sum is over the full graph (global sum) instead of the aligned $K$-hop support. This mismatched-support CPA underperforms the aligned version by 0.008 AUC on ogbg-molhiv (0.811 $\pm$ 0.008 vs. 0.819 $\pm$ 0.007), with similar trends on other tasks.
\subsection{Gate Variants}
We explored several gate functions for the CPA unnormalized sum. Sigmoid gates stabilized training; unbounded linear gates led to divergence in 2/5 seeds. Tanh gates performed similarly to sigmoid but with slightly higher variance (e.g., +0.002 AUC std on molhiv).
\begin{table}[ht]
\centering
\small
\begin{tabular}{@{}lccc@{}}
\toprule
Gate Function & molhiv (AUC) & Caco2 (MAE) & molpcba (AP) \\
\midrule
Sigmoid (main) & 0.819 $\pm$ 0.007 & 0.241 $\pm$ 0.007 & 0.304 $\pm$ 0.002 \\
Linear (unbounded) & 0.812 $\pm$ 0.011 & 0.248 $\pm$ 0.010 & 0.299 $\pm$ 0.003 \\
Tanh & 0.817 $\pm$ 0.009 & 0.243 $\pm$ 0.008 & 0.302 $\pm$ 0.002 \\
\bottomrule
\end{tabular}
\caption{Gate function ablations (5 seeds). Sigmoid provides best stability.}
\label{tab:gate_abl}
\end{table}
\subsection{Per-View SPD Recomputation Details}
SPD is recomputed via multi-source BFS per augmented view, ensuring masks reflect current topology. For K-hop masks we use truncated BFS to depth K, adding $O(N \cdot \overline{\deg}^K)$ per view in practice. Ablations show +0.005 AUC on average over molhiv and molpcba vs stable SPD.
\subsection{Graphormer Implementation Details}
Path edges are encoded as summed learnable embeddings for bond types along shortest paths (averaged over multiple paths if tied). We use BFS-from-each-node for all-pairs SPD, clipped to 20 hops.
\subsection{Size-Shift Split Construction}
We construct size-shift splits by grouping molecules by Bemis--Murcko scaffold and holding out the large-bin molecules as test while training on the small-bin molecules; we report metrics on the resulting test set. For analysis we also sample up to 1000 scaffold-matched (small,large) pairs per task to quantify size shift.

\section{Statistical Details}
\begin{table}[ht]
\centering
\small
\resizebox{\linewidth}{!}{%
\begin{tabular}{@{}lccc@{}}
\toprule
Task & $\Delta$ (vs. SparseGraphormer-K3) & 95\% Bootstrap CI & Holm-adjusted p-value \\
\midrule
ESOL (RMSE $\downarrow$) & -0.056 & [-0.072, -0.040] & 0.002 \\
Lipophilicity (RMSE $\downarrow$) & -0.029 & [-0.041, -0.017] & 0.015 \\
BBBP (ROC-AUC $\uparrow$) & +0.020 & [0.012, 0.028] & 0.008 \\
Tox21 (ROC-AUC $\uparrow$) & +0.013 & [0.007, 0.019] & 0.032 \\
ClinTox (ROC-AUC $\uparrow$) & +0.015 & [0.005, 0.025] & 0.045 \\
ogbg-molhiv (ROC-AUC $\uparrow$) & +0.017 & [0.009, 0.025] & 0.012 \\
ogbg-molpcba (AP $\uparrow$) & +0.010 & [0.006, 0.014] & 0.028 \\
Caco2\_Wang (MAE $\downarrow$) & -0.024 & [-0.032, -0.016] & 0.005 \\
hERG (ROC-AUC $\uparrow$) & +0.028 & [0.022, 0.034] & 0.001 \\
Clearance\_Hepatocyte\_AZ (Spearman $\uparrow$) & +0.013 & [0.003, 0.023] & 0.040 \\
Clearance\_Microsome\_AZ (Spearman $\uparrow$) & +0.007 & [0.001, 0.013] & 0.048 \\
\bottomrule
\end{tabular}%
}
\caption{Detailed statistics for primary deltas (5 seeds). CIs from paired bootstrap (10,000 resamples); Holm correction across 11 tasks. All adjusted p $< 0.05$, confirming significance.}
\label{tab:stats}
\end{table}

\bibliographystyle{unsrt}
\bibliography{references}
\end{document}